 \documentclass[10pt,conference]{IEEEtran}

\input epsf
\usepackage{graphicx}
\usepackage{multirow}
\usepackage{cite}
\usepackage{amsmath,amssymb,amsfonts}
\usepackage[ruled, linesnumbered]{algorithm2e}
\usepackage{graphicx}
\usepackage{textcomp}
\usepackage{tablefootnote}

\usepackage{bbm}
\usepackage{url}
\usepackage{amsthm}
\usepackage{xcolor}
\usepackage{multirow}
\usepackage{comment}
\usepackage{soul,color}
\usepackage[colorinlistoftodos,prependcaption,textsize=small]{todonotes}
\usepackage{mathtools}

\let\oldtextbf\textbf
\renewcommand{\textbf}[1]{\oldtextbf{\boldmath #1}}
\newtheorem{thm}{Theorem}
\newtheorem{lem}{Lemma}
\newtheorem{Definition}{Definition}
\def\BibTeX{{\rm B\kern-.05em{\sc i\kern-.025em b}\kern-.08em
    T\kern-.1667em\lower.7ex\hbox{E}\kern-.125emX}}

\makeatletter
\newcommand{\linebreakand}{%
  \end{@IEEEauthorhalign}
  \hfill\mbox{}\par
  \mbox{}\hfill\begin{@IEEEauthorhalign}
}
\makeatother
\begin{document}

%title{Cross-Smoothing Majority Certification Analysis for Patch Robustness for Deep Learning Models}

\title{
  A Majority Invariant Approach to Patch Robustness Certification for Deep Learning Models%
  \thanks{
  This work is supported in part by CityU grant (project no. 9678180). W.K. Chan is the contact author.
  }
}

\author{
% \IEEEauthorblockN{Anonymous Author(s)}
\IEEEauthorblockN{Qilin Zhou, Zhengyuan Wei, Haipeng Wang, and W.K. Chan}
\IEEEauthorblockA{Department of Computer Science, 
City University of Hong Kong\\
\{qilin.zhou, % @my.cityu.edu.hk, 
zywei4-c, %@my.cityu.edu.hk, 
haipeng.wang\}@my.cityu.edu.hk, wkchan@cityu.edu.hk}
% \and
% \IEEEauthorblockN{Zhengyuan Wei}
% \IEEEauthorblockA{Department of Computer Science \\
% City University of Hong Kong\\
% zywei4-c@my.cityu.edu.hk}
% \linebreakand
% % \centerline~
% \IEEEauthorblockN{Haipeng Wang}
% \IEEEauthorblockA{Department of Computer Science \\
% City University of Hong Kong\\
% haipewang5-c@my.cityu.edu.hk}
% \and
% \IEEEauthorblockN{W.K. Chan}
% \IEEEauthorblockA{Department of Computer Science \\
% City University of Hong Kong\\
% wkchan@cityu.edu.hk}
}
%+++++++++++++++++++++++++++++++++++++++++++

% use only for invited papers
%\specialpapernotice{(Invited paper)}

% make the title area
\maketitle
\begin{abstract}
Patch robustness certification ensures no patch 
within a given bound on a sample 
can manipulate a deep learning model to predict a different label.
However, existing techniques cannot certify samples that cannot meet their strict bars at the classifier level or the patch region level.
This paper proposes \textit{MajorCert}. \textit{MajorCert} firstly finds all possible label sets manipulatable by the same patch region on the same sample across the underlying classifiers, then enumerates their combinations element-wise, and finally checks whether the majority invariant of all these combinations is intact to certify samples.
\end{abstract}
% \vspace{-0.35cm}
\section{Introduction}
\label{sec:intro}
Patch adversarial attack on images, which modifies pixels in a patch region of a sample, is a critical threat in the real world \cite{brown2017adversarial}. 
Its defense techniques are under intensive study \cite{sharma2022adversarial} and recently received increasing attention from the SE community \cite{huang2023patchcensor, zhou2020deepbillboard}.
% Its defense techniques are under intensively studies in software engineering \cite{huang2023patchcensor,zhou2020deepbillboard,zhang2018deeproad,chiang2020certified,xiang2022patchcleanser,naseer2019local,hayes2018visible,mccoyd2020minority,xiang2021patchguard++,han2021scalecert,levine2020randomized,salman2022certified,li2022vip,chen2022towards,lin2021certified,xiang2021patchguard,metzen2021efficient,zhang2020clipped}.
%\todo{cite patch attack detection or repair (defense) paper published in icse/ase/tse/tosem to avoid reviewers to state the work out of scope.}
% \todo{I only find these two at present in tops.}
Empirical techniques \cite{naseer2019local, hayes2018visible} are based on heuristics, 
% incurring non-trivial false positives/negatives, 
whereas certified techniques \cite{ 
mccoyd2020minority,xiang2021patchguard++,huang2023patchcensor,han2021scalecert,li2022vip,chiang2020certified,levine2020randomized,salman2022certified,chen2022towards,lin2021certified,xiang2021patchguard,metzen2021efficient,zhang2020clipped,xiang2022patchcleanser} have guarantee on robustness and are emerging.
%\todo[inline]{Sorry, I think it should be no false negative for certified detection (said in many papers on certified detection) and no false positive for certified recovery. }) 
% \todo{ no. let x1 be a sample that can be a certifiable sample by the oracle. x1 may or may not be reported by technique A (so, x is deemed as true positive and false negative, respectively. Thus, A misses to report x in the latter case.  on the other hand, by definition, any x if reported is certifiably robust. so, x is never a false positive. You are right. What I means false negative in certified detection is false negative on adversarial detection in their definition. I suggest "report no false positive for robustness" since many certified detection paper like the newly TOSEM \cite{huang2023patchcensor} say they are no "false negative".}
The latter has two broad categories of techniques: certified recovery is a harder
% \cite{mccoyd2020minority,xiang2021patchguard++,huang2023patchcensor,han2021scalecert,li2022vip,chiang2020certified,levine2020randomized,salman2022certified,li2022vip,chen2022towards,lin2021certified,xiang2021patchguard,metzen2021efficient,zhang2020clipped,xiang2022patchcleanser}
% : certified recovery is a harder 
% and more complete 
problem than certified detection, which proves to recover the label of benign samples against a patch \cite{xiang2021patchguard}.
% In contrast, certified detection can only detect and warn patched samples with false positives. This  paper studies certified recovery.
% We study certified recovery in this paper.

In this paper, we present a novel technique \textit{MajorCert} for certified recovery.
% Its basic idea is to determine whether all combinations of all original and malicious labels of all ablations of a given sample from different base models overlapping with each patch region result in the same prediction label under the majority rule.
%\textit{MajorCert} is an ensemble-based defender.
% neither determines such safety margins nor treats ablations overlapping with the same patch region independently of one another in its region-sensitive analysis.
Existing techniques find a common safety margin between labels at the classifier level \cite{levine2020randomized,salman2022certified,li2022vip,chen2022towards,lin2021certified} or check whether a malicious label exists at the patch region level \cite{xiang2021patchguard,metzen2021efficient,zhang2020clipped}.
% \todo{add more quantification on existing work}. 
A common safety margin is easy to check but overestimated due to the insensitivity of patch location.
Checking at the patch region level gives a strictly tighter certification bound \cite{metzen2021efficient} but still fails if any malicious label exists, limiting its applicability.
%the checking in individual patch region is still intolerant to the appearance of malicious labels, 
This paper proposes \textit{MajorCert}.
\textit{MajorCert} finds the possible label sets manipulatable by the same patch region from the underlying classifiers, enumerates their combinations elementwise, and 
% checks whether the majority invariant on all these combinations is intact to certify samples.
% Differently, \textit{MajorCert} allows and collects these malicious labels from various underlying classifiers, enumerates their combinations patch region, and 
checks the \emph{majority invariant} (see Lemma~\ref{lem:malicious})
% consistence of majority 
to give a certification on a sample.
% It is finer in granularity, leading to more accurate certification outcomes.
 % It dynamically establishes an invariant condition for each patch region and exhausts all cases against the invariant.
% We refer to this consistency condition as the \emph{majority invariant} (see the condition in Lemma~\ref{lem:malicious}).
Satisfying the majority invariant implies the sample is certifiably robust and the label could be recovered even if any patch therein (see Sec. \ref{sec:MIA}).
% before any patch possibly appearing in the sample.
% under analysis.

Main contributions: 
(1) This paper is the first work to show the feasibility of certified recovery to certify a sample under the majority invariant condition even if malicious labels exist at the classifier level or the patch region level.
% all ablations overlapping with the same patch region produce 
% malicious labels under the majority invariant condition.
% changes in prediction labels of ablations overlapping with the same patch region.\todo{wrong claim} 
(2) It presents $\textit{MajorCert}$ with its theory and a case study.
% experiment.
% to show the feasibility.
% (3) The experiment shows its potential to outperform existing techniques in patch robustness.

The rest of the paper will revisit the preliminaries (\S\ref{sec:prelim}), present \textit{MajorCert} (\S\ref{sec:majorcert}) and its evaluation (\S\ref{sec:expt}), followed by reviewing related works (\S\ref{sec:relatedWork}) and concluding this work (\S\ref{sec:conclusion}).

\section{Preliminaries}
\label{sec:prelim}
\subsection{Image Classification Model, Attack, and Defense}

% An image sample $x$ has $w$ rows in width and $h$ columns in height.
An image sample $x$ has $w$ rows in height and $h$ columns in width.
A classifier $f$ accepts $x$ as input to predict its label $c$ (denoted by $f(x)$).
% A classifier $f$ accepts $x$ as input to predict its label $c$, where $c \in C = \{{c_1, \ldots, c_n}\}$ with a predefined order and $c_i < c_j$ if $i < j$.
% We let the prediction label $c$ for $x$ be $f(x)$. 
%
We define the indicator function
 $\mathbf{1}[z]$ to return 1 if the condition $z$ holds, otherwise 0.
We adopt the convention used in $\arg\max(\cdot)$ to return the label with a smaller label index for tie-breaking \cite{levine2020randomized}, modeled by the function $\mathbf{1}[c_i > c_j]$ where $c_i > c_j$ if indexes $i > j$.
% with common-use label indexes $i,j$.
% We let the prediction label $c$ for $x$ be $f(x)$. 
 % is satisfied and returns 0 if not the case.

 % $f(x) \neq c$.
% It indicates the prediction class $c$ as  1 and all other classes as 0.

%\todo[inline]{In \cite{xiang2021patchguard}, PG's attacker objective requires $f(x') \neq y$ instead of $f(x') \neq f(x)$. If $f(x) \neq y$, the attacker succeeds even if $f(x) = f(x')$. Its defense objective is $D(x) = D(x') = y$ for any \textbf{clean} sample $x$.}
%
A \textbf{patch region} in an image is represented by a binary matrix $p \in [0,1]^{w \times h}$ where all the elements within the region are set to 1, otherwise 0. 
We follow prior works %in patch robustness against patch adversarial attacks 
\cite{levine2020randomized,xiang2021patchguard} 
to focus on cases where the patch region is a square of side $m$ 
(as a conservative estimate by a defense).
% , which means an all-ones sub-matrix is in $p$ with size $m$ 
We regard $m$ as the size parameter of $p$.
We define two samples $x_1$ and $x_2$ \textbf{differ} by $p$ if and only if $(J-p) \cdot x_1 = (J-p) \cdot x_2$ and $x_1 \neq x_2$ where $\cdot$ is the element-wise product operator and $J$ is an all-ones matrix \cite{horn2012matrix} of the same
dimension as $p$.
%
% Given  a sample $x$,
% with its ground truth label $y$, 
An attacker aims to find a sample $x_2$ such that $f(x_2) \neq f(x_1)$.
% The content of the adversarial patch in $x'$ is $p \cdot x'$.
% We follow prior works %in patch robustness against patch adversarial attacks 
% \cite{levine2020randomized,xiang2021patchguard} 
% to focus on cases where one all-ones sub-matrix is in $p$ with a square of side $m$ (as a conservative estimate by a defense).
%each adversarial patch represents one square region
 % is a square of side $m$
 %  (as a conservative estimate by a defense). 
% We refer to the location of the top-left corner of a square region $p$ as the patch's location $l$.
%
%Given this threat model, 
A defense aims to design a defender $D$ such that $D(x_1) = D(x_2) = f(x_1)$ for every patched version $x_2$ of $x_1$.
% for $x'$ differing from 
% at any location $l$ 
% applied to the sample 
% $x$ by a patch region $p$ of  size $m$.
% On the test sample $x$, $D$ will produce the same label as $D(x)$.
%
% \todo[inline]{Here, corresponding to patch robust accuracy, it is better to define $x$ as a \textbf{certifiably robust sample}?---- you need ``certifiably robust'' in many places of this paper ok but is it need to be $x$ --- you can change it to make it technically correct. I have no clue/sense on which side is correct.}
%
We refer to such $x_1$ as a \textbf{certifiably robust sample} on $D$.
% (i.e., satisfying the above constraint for any square patch region of side $m$ represented by $p$).
As noted in \cite{xiang2021patchguard},
certified recovery 
%Patch robustness certification 
needs to guarantee the recovery of the label of $x_1$ from $x_2$, which is harder than merely detecting $x_2$ as adversarial.
% \todo{PG's requires $F(x) = F(x') = y$}
% \vspace{-0.03cm}
\subsection{Row, Column, and Block Ablation Strategies}
An \textbf{ablation region} is represented by a binary matrix $\psi \in [0,1]^{w \times h}$ like patch region.
% where each element within the region is  % set to 
% 1, otherwise 0.
Applying $\psi$ on a sample $x$ creates an \textbf{ablation} $b$ (i.e., $b = \psi \cdot x$).
Existing work \cite{levine2020randomized} studied three types of ablation regions: row, column, and block.
A row ablation region $r_i$ has width $w$ and height $s_r$, spanning from rows $i$ to $[i + (s_r-1)] \mod h$.
A column ablation region $c_j$ has height $h$ and width $s_c$, spanning from columns $j$ to $[j + (s_c-1)] \mod w$.
A block ablation region $b_{i,j}$ is a square ablation region of side $s_b$, spanning from rows $i$ to $[i + (s_b-1)] \mod h$ and columns $j$ to $[j + (s_b-1)] \mod w$. 
Applying the above three rules for all $0 \leq i < h$ and $0 \leq j < w$ produces three sets of ablation regions $\Psi_{a_r}$,$\Psi_{a_c}$, and $\Psi_{a_b}$  respectively, and applying  $\Psi_{a_r}$,$\Psi_{a_c}$, and $\Psi_{a_b}$ 
on a sample $x$ produces three sets of ablations
% which apply to a sample $x$ producing three corresponding sets of ablations, 
denoted by $\mathbb{B}(a_r, x)$, $\mathbb{B}(a_c, x)$, and $\mathbb{B}(a_b, x)$, respectively.
We refer to these three \textbf{strategies} to create ablations as $a_r$, $a_c$, and $a_b$, and use the notation $\mathbb{A}$ to represent a set of ablation strategies.

We let \textbf{$f^a$} be a classifier trained on the set $\cup_{x \in T} \mathbb{B}(a, x)$ using the strategy $a \in \mathbb{A}$ and a training dataset $T$, and let the indicator function $f^a_c(x)$ yield 1 if $f^a(x) = c$, otherwise 0.

\subsection{DeRandomized Smoothing (DRS)}
\label{sec:DRS}
% Image ablation systematically obscures a part of an image $x$ to produce a set of image variants (ablated image) by a rule \cite{levine2020randomized, salman2022certified}.  Exemplified ablation rules include column ablation, row ablation, and block ablation, which are popularly studied by the existing works on certified defenders \cite{levine2020randomized}. \todo{need to introduce them as your paper needs these three types of ablation} We denote $B_i(x)$ as the set of ablated images produced from applying an ablation rule $i$ on the image $x$. If the rule $i$ is clear or it can be arbitrary, for brevity, we simply use the simplified notation $B(x)$ instead.

DRS \cite{levine2020randomized} is an ensemble-based technique.
% to certify whether a sample is certifiably robust.
%certify a sample that its prediction label is not affected by any patch appearing in the sample.
%to certifiably defense against patch attack, where 
% \todo[inline]{very unclear to readers. why is there a base classifier $g$? what is its relation to the ablated versions of the training dataset of an image classification task. how the smoothed classifier is constructed? why don't use $g$ directly? }
% For a given ablation strategy $a \in \mathbb{A}$, it trains a  classifier $f^a$ on the ablation set $\cup_{x \in T} \mathbb{B}(a, x)$ where $T$ is the training dataset of the image classification task.
Given a classifier $f^a$,
% trained on the ablation strategy $a$.
it accepts a test sample $x$ as input, produces the ablation set $\mathbb{B}(a, x)$ from $x$, predicts a label $c$ for $x$ using the function 
$D^a_\textit{DRS}(x)$, 
% (by applying the majority rule \cite{may1952set} on the ensemble group
% $\langle f^a(b) \mid b \in \mathbb{B}(a, x) \rangle$)
and reports whether $x$ is certifiably robust.  
% We note that May's theorem \cite{may1952set} shows the majority rule is the only ``fair" decision rule.
\[
D^a_\textit{DRS}(x)=\mathop{\arg\max}\nolimits_{c} n_{c}(x)
% \]
\text{ where }
% \[
n_{c}(x)=\sum\nolimits_{b\in \mathbb{B}(a, x)}f^a_c(b)
\]
DRS proves that $x$ is certifiably robust if $D^a_\textit{DRS}(x)=c$ and 
$
% D^a_\textit{DRS}(x)
n_{c}(x)
\geq 2\Delta +
\mathop{max}_{c'\neq c}[n_{c'}(x)+\textsc{1}[{c>c'}]]
$,
where 
% $\textsc{1}_{c>c'}$ is the indicator function to break the tie deterministically by label index during the final classification of the $\arg\max(.)$ function, and 
$\Delta = m+s_r(s_c)-1$ for row (column) ablation and 
$\Delta = (m+s_b-1)^2$
for block ablation \cite{levine2020randomized}.
DRS directly computes $\Delta$ based on the size parameter of patch region $m$ and the size parameter of the ablation region $s_r$,$s_c$,$s_b$ that creates its $f^a$ for checking.
% , and then checks whether the above condition is met.
If the above condition holds, it reports $x$ is certifiably robust.

% certification analysis calculates an upper bound on the number of ablations $n_{\textit{max}}$ in the ablation set $\mathbb{B}$ for a given ablation strategy 
% % (with \emph{arbitrary} prediction labels)  
% that an arbitrary patch region can simultaneously overlap with them  and checks whether 
% reducing the number of elements  in   $\mathbb{E}(x)$ for any label by $n_{\textit{max}}$ changes the majority rule result for the current sample $x$.
% If the majority rule result remains unchanged, it reports $x$ certifiably robust.

% Later techniques \cite{xiang2021patchguard,metzen2021efficient} examine the overlapping requirement for each specific sample among elements in $\mathbb{B}(a, x)$ for a given ablation strategy $a \in \mathbb{A}$. They enumerate the number of elements $n_{\textit{\emph{($x$,$p$)}-specific}}$ in $\mathbb{B}(a, x)$ overlapping with each patch region $p$ and checks $\mathbb{E}(x)$ with $n_{\textit{\emph{($x$,$p$)}-specific}}$ instead of $n_{\textit{generic}}$. Moreover, apart from counting prediction label instances, they also develop variants of $\mathbb{E}(x)$ and the comparison procedure, such as taking the maximum after aggregating the prediction probabilities or confidence of the predictions on $\mathbb{B}$ instead of taking the majority vote rule and applying heuristics to exclude extreme values before the comparison procedure to fine-tune the detector performance for a given dataset. Their procedures can certify some samples that DRS fails.

% \todo[inline]{ The paper needs a running example to illustrate the method and each lemma and theorem. }

\begin{algorithm}[t]
\caption{MajorCert}
\label{alg:MajorCert}
\SetKwInOut{KwIn}{Input}
\SetKwInOut{KwOut}{Output}
\KwIn{$\mathbb{A}$ $\gets$ set of ablation strategies, \\$\mathbb{M}$ $\gets$ map of pretrained models ($\mathbb{A}$ as keys), \\$\mathbb{R}$ $\gets$ map of ablation region sets ($\mathbb{A}$ as keys), \\$x$ $\gets$ test sample, \\
      }
\KwOut{$y$ $\gets$ prediction label, \\
${cert}$ $\gets$ certification tag \\}

\textit{V} $\gets$ \{\}\\
  \ForEach{$a \in \mathbb{A}$}{
    $\mathbb{B}(a, x)$ $\gets$ $\langle b | b = x \cdot \psi \land \psi \in \Psi_{a} \land \Psi_{a}=\mathbb{R}$.get$(a) \rangle$
    % GenerateAblations($x$,$a$) 
    \\
%    $f^a = \mathbb{M}$.get$(a)$ \\
% $V_{a}$ $\gets$ GenerateVote($f^a$, $\mathbb{B}(a, x)$)\\
\textit{V}[$a$]  $\gets$ $\langle f^a(b) \mid b \in \mathbb{B}(a, x) \land f^a = \mathbb{M}$.get$(a) \rangle $ \\
%\textit{V}[$a$] $\gets$ $V_a$\\
}

$y$ $\gets$ \textsc{ComputeLabel}(\textit{V})\\
% \textit{MLDict} $\gets$ \textit{MaliciousCheck}(\textit{votingDict})\\
${cert}$ $\gets$ \textsc{MajorityCertification}(\textit{V})\\
\If{${cert}$ = False}{${cert}$ $\gets$ {\small \textsc{MajorityInvariantCertification}(\textit{V})}}
\Return{$y$, ${cert}$}
%\vspace{-0.5ex}
\end{algorithm}

\subsection{Conservative upper and lower bounds on a label}

Suppose two samples $x_1$ and $x_2$ differ by $p$, and $b$ is an ablation produced by $\psi$ in the ablation set $\mathbb{B}(a, x_1)$.
If $p$ overlaps with $\psi$ (i.e.,  $\psi \cdot p\neq O$ where $O$ is a zero matrix of the same dimension as $p$), 
then the pair of labels returned by $f^a( x_1 \cdot \psi)$ and $f^a( x_2 \cdot \psi)$ may be different, otherwise keeps the same. We define the subset of $\mathbb{B}(a, x_1)$ in which $\psi$ of each $b$ doesn't overlap with $p$ (which also means $b \cdot p = O$) as $X$.
Thus, the number of ablations in $X$ predicted to a label $c$ represents a \textbf{lower bound $\underline{n_{c}}(x_1, a, p)$} on the number of ablations of $x_2$ predicted to $c$ by $f^a$,  defined as:
$\underline{n_{c}}(x_1,a,p)=\sum_{b\in \mathbb{B}(a,x_1)\land  b \cdot p = O} f^a_c(b)$.
Similarly, the corresponding \textbf{upper bound $\overline{n_{c}}(x_1,a, p)$} is counted as the number of ablations in $X$ predicted to $c$ by $f^a$ plus the total number of remaining ablations in $\mathbb{B}(a, x_1)/X$, defined as: $\overline{n_{c}}(x_1,a, p)=\sum_{b\in \mathbb{B}(a,x_1)\land  b \cdot p = O} f^a_c(b)+\sum_{b\in \mathbb{B}(a,x_1) } \textsc{1}[{b \cdot p\neq O}]$.

% Let $N_1$ be the multiset  $\{ f^a( b) | b \in \mathbb{B}(a, x_1) \}$.  In the worst case, all ablation regions overlap with $p$ are unsafe to produce their labels, and let $N_2$ be the mutiset $\{ f^a( x \cdot \omega_i) | \omega_i \in \mathbb{B}(a) \land \omega_i \cdot p \neq 0 \}$. 

\section{MajorCert}
\label{sec:majorcert}
\subsection{Overview}
\label{sec:overview}
\textit{MajorCert} consists of label prediction and two certification analyses. In the main algorithm (Alg. \ref{alg:MajorCert}), \textit{MajorCert} first creates an empty map $V$ to contain prediction results and then generates the prediction label for each ablation $b$ under each ablation strategy $a$ (lines 1--5) for the input sample $x$. 
It then predicts the final label (line 6, see Sec. \ref{sec:Prediction_Label}) and checks whether over half of the DRS defenders (one for each ablation strategy) can certify $x$ (line 7, see Sec. \ref{sec:Majority_Certification_Analysis}).
If this is not the case (line 8), it conducts the majority invariant certification analysis (line 9, see Sec. \ref{sec:Majority_Invariant_Analysis}) to generate a certification tag. Line 11 returns the label and the certification tag.
Fig. 1 illustrates how \textit{MajorCert} compares to existing techniques.
% can certify samples unable to be certified by existing work.

% \todo[inline]{give an overview of the algorithm and the use of thm and lemma} [height=2cm]
\begin{figure}[t]
\centering
\includegraphics[width=8.5cm]{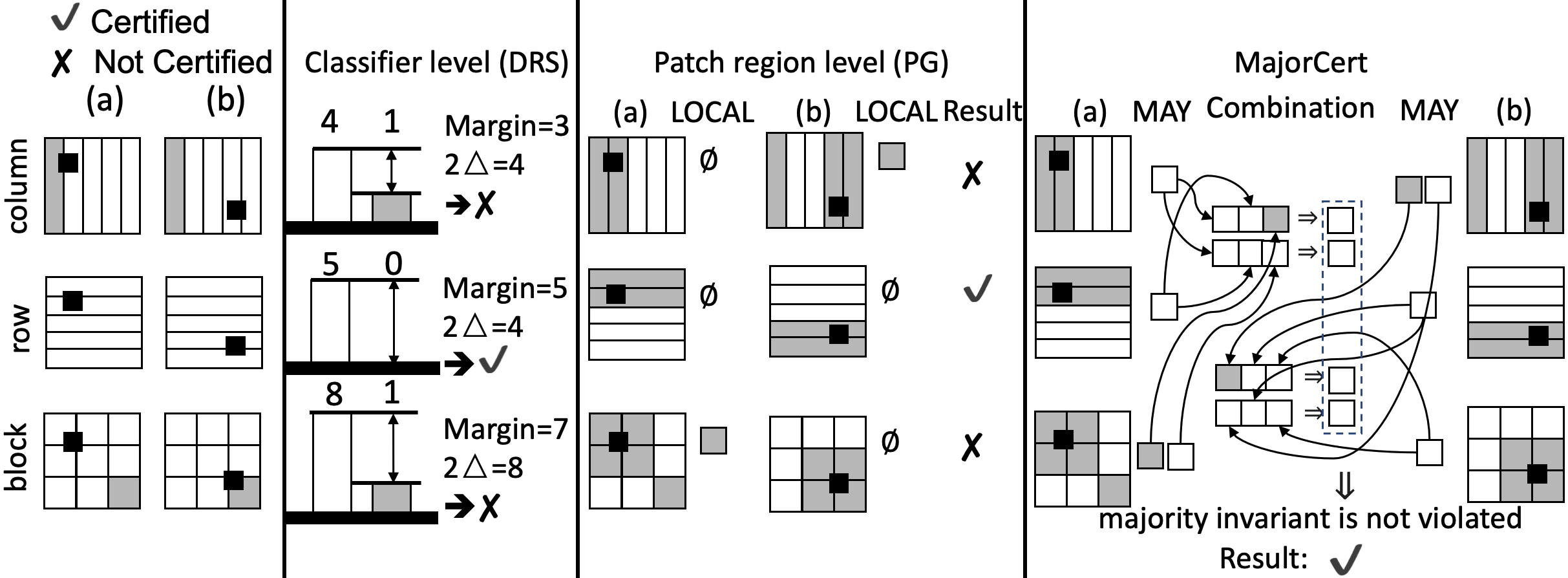}
\caption{MajorCert 
% with majority invariant analysis 
compares to DRS \cite{levine2020randomized} and PG \cite{xiang2021patchguard} on a sample with its ablations' predictions when using the column, row, and block ablation strategies from top to bottom. Neighbour ablations will overlap practically but are omitted here.
The ablations in white and grey stand for prediction to two different class labels.
% and the black square stands for a possible patch region. 
Scenarios (a) and (b) stand for two situations with a patch region (black square) on the sample. 
See Section III.B for the definitions of \textsc{LOCAL} and \textsc{MAY}.  In the column and block ablation cases, non-empty \textsc{local} sets fail PG, and insufficient margins fail DRS. \textit{MajorCert} certifies the sample even if the \textsc{May} sets contain local-malicious labels and if the majority of DRS (PG) defenders fail to do so.
}
\label{fig:picture001}
\vspace{-3ex}
\end{figure}

% MajorCert is a novel deep learning framework to provide patch robustness against patch attack. A set of pretrained base models (e.g. DRS\cite{levine2020randomized}) with different smoothing schemes is needed to construct a more robust ensemble model in MajorCert. Firstly, MajorCert computes the output label of each ablation from each ablation and then aggregates them for the prediction output of MajorCert (see......). Then, MajorCert checks for sets of malicious labels for each smoothing scheme (see......). Next, MajorCert performs the first certification analysis on the sample, called \textit{static majority certification analysis} (see ....). If the certification result is true, then MajorCert can confirm that the input sample is a certified robust sample and output the result; if not, MajorCert performs the second certification analysis, \textit{dynamic majority certification analysis}. we note that \textit{dynamic majority certification analysis} is not a simple aggregation of results of \textit{dynamic certification analysis} from different schemes, meanwhile, \textit{dynamic majority certification analysis} ensures that certified robust samples obtained from it are a superset of those obtained from a simple aggregation of \textit{dynamic certification analysis} from different schemes, which is an insight of MajorCert (see ....). Finally, MajorCert outputs this sample's prediction label and certification result.

\subsection{Certification Analysis}
% \vspace{-1ex}
\subsubsection{Prediction Label}\label{sec:Prediction_Label} \textit{MajorCert} predicts a label via the function $D_{\textit{MC}}(x)$, where $x$ is a test sample, and all classifiers $f^a$ for all $a \in \mathbb{A}$ share the same training dataset $T$.
% the same hyperparameters $m$ and $s$ and
%\todo[inline]{$m$ doesn't affect prediction label and $s$ are different between column/row and block so I delete them.}
\begin{equation}
\label{def:D_mc}
%\begin{split}
D_{\textit{MC}}(x)  = \mathop{\arg\max}\limits_{c}\sum\limits_{a\in \mathbb{A}}%
% f^a_c(x) \text{ where } 
% \\
% f^a_c(x)  = 
\textsc{1}[[\mathop{\arg\max}\limits_{c'}\sum\limits_{b\in \mathbb{B}(a,x)} f^a_{c'}(b)]=c]
%\end{split}
\end{equation}
The function \textsc{ComputeLabel}($\cdot$) in Alg.~\ref{alg:MajorCert} puts the label having the largest number of instances in $V[a]$ for each $a$ into an initially empty map, followed by finding the label having the largest number of instances in this map, i.e., computing $D_{\textit{MC}}(x)$ as
$\mathop{\arg\max}_{c}\sum_{a\in \mathbb{A}}
\textsc{1}[[\arg\max_{c'} \sum_{v \in V[a]} 1[v = c']]=c]
$, which is equivalent to Eq.~(\ref{def:D_mc}).

%votes in $V$ from each classifier produces one label with highest vote and \textsc{ComputeLabel} output the label with the most strategy produce (e.g. $D_{\textit{MC}}(x)$) as prediction label $y$.
\begin{comment}
Firstly, we recall Eq. \ref{majority}, which is the way many previous certified techniques \cite{levine2020randomized}\cite{xiang2021patchguard}\cite{salman2022certified} construct the smoothed model. Based on the smoothed model define in Eq. \ref{majority}, we define an ensemble smoothed model $G(x)$ as following:
\begin{equation}
G(x)=\mathop{argmax}\limits_{c}\sum\limits_{s\in \mathbb{S}}F_s^c(x)\\
\label{majoritysmoothed}
\end{equation}
where $F_s^c(x)$ denotes that the smoothed model $f(x)$ with smoothing scheme $s$ outputs label $c$. We follow the initial constrain in Eq. \ref{majority} to break the tie by the label index. Briefly, $G(x)$ aggregates prediction labels from smoothed models by calculating their majority. We believe that the diversity of smoothing schemes and smoothed models can provide more robustness to our model.
\end{comment}

\subsubsection{Majority Certification Analysis}\label{sec:Majority_Certification_Analysis}
Like \cite{xiang2021patchguard},
our majority certification analysis is built atop DRS \cite{levine2020randomized}. We regard patch size $m$ as known by all and omit it.
Given a classifier $f^a$, an input $x$, and a label $c$, we let $\theta^a_c(x)$ denote the indicator function that returns 1 if both $D_\textit{DRS}^a(x)$ = $c$ and the defender $D^a_\textit{DRS}$ for the ensemble 
$\mathop{\arg\max}_{c}\sum_{b\in \mathbb{B}(a,x)} f^a_c(b)$
%classifier $f^a$  
reports 
$x$ as certifiably robust, otherwise returns 0.

% In essence, 
Theorem \ref{thm:static_ensemble} states that if more than half of the defenders $D_\textit{DRS}^a$ for all $
a \in \mathbb{A}$ certify the sample for the same label, then our technique also certifies the sample. It is a direct consequence of using an ensemble over a set of defenders.
% The smoothed models $f(x)$ aggregated by $G(x)$ have already had the above Theorems to guarantee their outputs. We denote ${F_{s}^{c}}_{SCert}(x)$ as the certification status from Theorem \ref{thm:static} on the smoothed models $F$ with the smoothing scheme $s$ for label $c$ upon the sample $x$, where 1 is for a certified robust sample, and 0 is for not. Based on Theorem \ref{thm:static}, we carry out the first guarantee on ensemble model $G(x)$:

% \vspace{-1ex}
\begin{thm}[majority certification analysis]\label{thm:static_ensemble}
A sample $x$ is certifiably robust on $D_{\textit{MC}}(\cdot)$ if $\exists c$, s.t. $\sum_{a\in \mathbb{A}}{\theta^a_c(x) > {|\mathbb{A}|}/{2}}$.
% $\sum_{a\in \mathbb{A}}{\theta^a_c(x) > |\mathbb{A}|}/{2}$.
\end{thm}
% \vspace{-0.05cm}
\vspace{-0.3cm}
\begin{proof}
 %Let  $\textsc{1}_{z}$ be a function that returns 1 if the condition $z$ is satisfied and returns 0 if not the case.
% The theorem is a direct consequence of the basic majority rule.
Suppose $x_1$ and $x_2$ differ by $p$ and $c_1$ is a label satisfying the condition $\sum_{a\in \mathbb{A}}{\theta^a_{c_1}(x_1) > {|\mathbb{A}|}/{2}}$. For $x_1$, we know if $\theta^a_{c_1}(x_1) = 1$, then $ D^a_{\textit{DRS}}(x_1) = c_1$. So from the given condition, we know over half of the classifiers return $c_1$. Thus, $D_{\textit{MC}}(x_1)$ = $\mathop{\arg\max}_{c} [\sum_{a\in \mathbb{A}} 
 \textsc{1}[{ D^a_{\textit{DRS}}(x_1) = c}]]$ = $c_1$. For $x_2$, we know if $\theta^a_{c_1}(x_1) = 1$, then $f^a(x_2)=f^a(x_1)$ for any patched version $x_2$ of $x_1$. Thus $\sum_{a\in \mathbb{A}}[\textsc{1}[{D^a_{\textit{DRS}}(x_2) = c_1}]]$ $\geq$ $\sum_{a\in \mathbb{A}}{\theta^a_{c_1}(x_1)} > {|\mathbb{A}|}/{2}$, where over half of the classifiers must return $c_1$ for $x_2$. Finally, we have 
 $D_{\textit{MC}}(x_2)$ = $\mathop{\arg\max}_{c} [\sum_{a\in \mathbb{A}} 
 \textsc{1}[{ D^a_{\textit{DRS}}(x_2) = c}]]$ = $c_1$.\end{proof} 
 % = $D_{\textit{MC}}(x_1)$.
% we must have $c_1 = $ $\mathop{\arg\max}_{c} \sum_{a\in \mathbb{A}} 
%  \theta^a_c(x)$.
% So, we have $D_{\textit{MC}}(x)$ =
%  $\mathop{\arg\max}_{c} \sum_{a\in \mathbb{A}} 
%  \textsc{1}[D^a_{\textit{DRS}}(x) = c]$ = 
%   $\mathop{\arg\max}_{c} \sum_{a\in \mathbb{A}} 
%  \theta^a_c(x)
%  = c_1$.
 % applied to $x$ to produce a patched version $x'$ of $x$.

 % !!!!!!!!!
  % !!!!!!!!!
  
%  Also, we have 
%  $D_{\textit{MC}}(x_2)$ =
%  $\mathop{\arg\max}_{c} \sum_{a\in \mathbb{A}} 
%  \textsc{1}[{ D^a_{\textit{DRS}}(x_2) = c}]$
% = 
%  $\mathop{\arg\max}_{c} \{
%  \sum_{
%  %\substack{
%  a\in \mathbb{A} %\\  
%  \land c = c_1
%  %}
%  }  \textsc{1}_{ D^a_{\textit{DRS}}(x_2) = c'}$, 
%  $\mathop{\max} \{$
%  $
%   % \sum\limits_{
%   \sum_{
%    %\substack{
%      a\in \mathbb{A} \land % \\
%     c \neq c' 
%   %} 
%   }$ $
%   \textsc{1}_{ D^a_{\textit{DRS}}(x')}  = c
%  \}
%  \}$
%  =  $\mathop{\arg\max}_{c} 
%  $ $
%  \sum_{
%  %\substack{
%  a\in \mathbb{A} 
%  %\\  
%  \land c = c'
%  %}
%  }  \textsc{1}_{ D^a_{\textit{DRS}}(x') = c'}$
%  (because more than half of all classifiers return $c'$), which is further equal to 
%  $\mathop{\arg\max}_{c} 
%  \sum_{
%  %\substack{
%  a\in \mathbb{A} 
%  %\\ 
%  \land c = c'
%  %}
%  }  \textsc{1}_{ D^a_{\textit{DRS}}(x) = c'}$ 
%  % (because DRS reports that $x$ is certifiably robust by its classifier and the label is $c'$),
%  %which is equal to 
%  = $c'$.

\vspace{-1ex}
The function \textsc{MajorityCertification}(.) in Alg.~\ref{alg:MajorCert} 
counts the number of instances of each label in $V[a]$, checks whether the DRS's condition of patch robustness certification (see Sec.~\ref{sec:DRS}) for each ablation strategy $a$ holds. 
%based on the size of a patch region.
% which outputs \textit{True} if the condition holds, otherwise \textit{False}. 
It then reports \textit{True} if the condition in Theorem~\ref{thm:static_ensemble} holds, otherwise \textit{False}.

\subsubsection{Majority Invariant Analysis}\label{sec:Majority_Invariant_Analysis}
\label{sec:MIA}
\textit{MajorCert} checks label combinations, 
% predicted by ablations (some overlapping with a given patch region).
% the possible  labels of ablations that overlap a given patch region at their combination level. 
% We refer to such a label as a \emph{local-malicious label}.
which is finer in granularity than
% The granularity level of this checking procedure is finer than the one 
checking at the patch region level \cite{xiang2021patchguard} or the classifier level \cite{levine2020randomized}, thereby thus producing a more accurate analysis.
% of patch robustness.

Let the function $g^a(\cdot)$ represent the inner $\arg\max$ term (= $\mathop{\arg\max}_{c'}\sum_{b\in \mathbb{B}(a,x)} f^a_{c'}(b)$) in Eq.~(\ref{def:D_mc}) for each $a \in \mathbb{A}$. Suppose $x_1$ and $x_2$ differ by $p$ and $g^a(x_1)=c$. A {local-malicious label} $c_l$ for $g^a(\cdot)$ and $x_1$ is 
% produced by $g^a(x_1)$ for a given ablation strategy $a \in \mathbb{A}$ 
caused by those ablations which are overlapped with $p$. The idea is to check whether the lower bound on the number of ablations predicted to label $c$ is smaller than the upper bound on the number of ablations predicted to another label (e.g., $c_l$).
 %saving the tie case. 
 If that is the case, $g^a(x_2)$ may return $c_l$ instead of $c$. 
Definition \ref{def:lm-label} captures this property.

\begin{Definition}[local-malicious label]
\label{def:lm-label}
Let $x_1$ be a sample, 
% and 
% the function  $g^a(x)$ = $\mathop{\arg\max}_{c}\sum_{b\in \mathbb{B}(a,x)} f^a_c(b)$. 
suppose $g^a(x_1)$ returns a label $c$.
A label $c_l \neq c$ is called a \emph{local-malicious label} of $x_1$ for $g^a(\cdot)$
% denoted by $\delta(g^a, p, x_1)$
with respect to a given $p$
% location $l$ for the sample $x$ on smoothed classifier $F$ 
if and only if
%\begin{equation}
%\begin{aligned}
 $\underline{n_{c}}(x_1,a, p) < \overline{n_{c_l}}(x_1,a, p)+\textsc{1}[{c>c_l}]$, where $\textsc{1}[{c>c_l}]$ is for tie-breaking. We let \textsc{local}$(x,a,p)$ denote the set of local-malicious labels of $x$ for $g^a(\cdot)$ and $p$.
%\end{aligned}
%\end{equation}
\end{Definition}

 % In general, we may find a set of labels whose upper bounds on the numbers of ablations are each larger than the lower bound on the number of ablation predicted to label $c$.
% More than one  $c' \in C$ may satisfy the condition for the same $c$ in Definition 1.
% We let \textsc{local}$(g^a(x), p)$ denote the set of local-malicious labels of $x$ for the pair of $g^a(\cdot)$ and $p$.
% 
%with $x'$ as input 
\vspace{-1ex}
Suppose $x_1$ and $x_2$ differ by $p$. Then, $g^a(x_2)$ should return a label in the set \textsc{local}$(x_1,a,p)$
or the original label $g^a(x_1)$.
\vspace{-1ex}
%
% represents an overapproximation of the prediction labels that the inner $\arg\max$ term of Eq.~(\ref{def:D_mc}) may return even if $x$ has been modified by a patch within the patch region $p$. So, any patched version $x'$ of $x$ will eventually produce a label in \textsc{local}$(g^a(x), p)$. Thus, if we know a label $c'$ is never the prediction label of $x$ returned by $g^a(.)$ nor a local-malicious label of $x$ in \textsc{local}$(g^a(x), p)$), then it will also never be the prediction label of a patched version $x'$ of $x$ (where $x$ and $x$ differ by $p$) returned by $g^a(.)$.\todo{need to write better. not intuitive yet.}
Lemma \ref{lem:no_malicious} captures this relation.

% Recall that each patch region has a unique location. For the same function $g^a(.)$, a sample may have different local malicious labels for different patch regions, meaning that $x$ may have different robustness at different patch regions. Recall our detector has a series of such functions: $g^a(.)$ for all $a \in A$. Our insight is to use the contradiction of local malicious labels across these functions to @@@

%The local malicious label is a concept that depends on the locations, which means we may have different local malicious labels for different locations for one sample. The local malicious label shows the different robustness of different places on the sample, which can be used to construct a more robust certification rule. Based on this definition, 
% We first prove 2 lemmas for our main result (Theorem 2).

% \todo[inline]{I delete the sub-condition ``$\land$ $g^a(x_1) \neq c'$'' as the proof does not need this component.}
\begin{lem}\label{lem:no_malicious}
Let $x_1$ be a sample and $p$ represent a patch region. Suppose $g^a(x_1)$ = $c_1$.
% If the condition of $x_1$ and $x_2$ differing by $p$ holds, then the condition $g^a(x_2) \in \textsc{local}(x_1,a,p)\cup \{c_1\}$ also holds. 
If $x_1$ and $x_2$ differ by $p$, then the condition $g^a(x_2) \in \textsc{local}(x_1,a,p)\cup \{c_1\}$ holds.
\end{lem}
\vspace{-0.32cm}
\begin{proof}
Suppose 
$g^a(x_2) = c_2 \notin \textsc{local}(x_1,a,p)\cup \{c_1\}$.
By Definition~\ref{def:lm-label},
we know the condition 
$\underline{n_{c_1}}(x_1,a,p) \geq \overline{n_{c_2}}(x_1,a, p)+\textsc{1}[c_1>c_2]$ holds, which can be further rewritten into $\sum_{b\in \mathbb{B}(a,x_1)\land  b \cdot p = O} f^a_{c_1}(b) \geq \sum_{b\in \mathbb{B}(a,x_1)\land  b \cdot p = O} f^a_{c_2}(b)+\sum_{b\in \mathbb{B}(a,x_1) } \textsc{1}[{b \cdot p\neq O}]+\textsc{1}[{c_1>c_2}]$. Since $x_1$ and $x_2$ differ by $p$, the condition $(J-p) \cdot x_1 = (J-p) \cdot x_2$ holds, which implies $\sum_{b\in \mathbb{B}(a,x_1) \land  b \cdot p = O} f^a_{c}(b) = \sum_{b\in \mathbb{B}(a,x_2)\land  b \cdot p = O} f^a_{c}(b)$ holds for all $c$. 
%
% and thus, $\sum_{b\in \mathbb{B}(a,x_1) \land  b \cdot p \neq 0} f^a_{c}(b) = \sum_{b\in \mathbb{B}(a,x_2)\land  b \cdot p \neq 0} f^a_{c}(b)$ for all $c \in C$. 
We then also have $\sum_{b\in \mathbb{B}(a,x_1) } \textsc{1}[{b \cdot p\neq O}] = \sum_{b\in \mathbb{B}(a,x_2) } \textsc{1}[{b \cdot p\neq O}]$, which implies $\sum_{b\in \mathbb{B}(a,x_2)\land  b \cdot p \neq O} f^a_{c}(b) \leq \sum_{b\in \mathbb{B}(a,x_2) } \textsc{1}[{b \cdot p\neq O}] = \sum_{b\in \mathbb{B}(a,x_1) } \textsc{1}[{b \cdot p\neq O}]$ holds for all $c$. So, we have 
$\sum_{b\in \mathbb{B}(a,x_2)\land  b \cdot p = O} f^a_{c_1}(b) \geq \sum_{b\in \mathbb{B}(a,x_2)\land  b \cdot p = O} f^a_{c_2}(b)+\sum_{b\in \mathbb{B}(a,x_2)\land  b \cdot p \neq O} f^a_{c_2}(b)+\textsc{1}[{c_1>c_2}]$, which implies
% The condition means that the lower bound on the number of instances for label $c$ should be larger than the upper bound on the number of instances for label $c'$ for the pair $x_1$ and $p$.
% Thus, every pair of $x_1$ and its patched version $x_2$ differing by the region $p$ should satisfy the condition
% which means there are always
%\begin{equation}
%  \forall p \in A,
$n_{c_1}(x_2) \geq n_{c_2}(x_2)+\textsc{1}[{c_1>c_2}]$.
Then, we have $g^a(x_2)\neq c_2$ from the definition of $g^a(\cdot)$,  contradicting to $g^a(x_2)=c_2.$
%\end{equation}
% noted that $f(x)=\mathop{argmax}\limits_{c}n_{c}(x)$. Finally, 
% So, we have $g^a(x_2) \neq c'$.
% \begin{equation}
% \forall p \in A,\ f(x_{p(l,m)})\neq c'
% \end{equation}
\end{proof}

\vspace{-0.2cm}
We define the set $\mathbb{E}_p(x)$ (for each patch region represented by $p$ in each index) to contain all combinations of the following sets~$\textsc{May}(x, a, p)$ = \textsc{local}$(x,a,p)\cup \{g^a(x)\}$ for all $a \in \mathbb{A}$ elementwise: 
$
\mathbb{E}_p(x) = \{ \langle
c \in \textsc{May}(x,a,p)
\rangle_{a \in \mathbb{A}}
\}$. %
 (E.g., Suppose $\mathbb{A} = \{a_r, a_c, a_b\}, \textsc{May}(x, a_r, p) = \{c_1, c_2\}, \textsc{May}(x, a_c, p) = \{c_2\}  \text{, and }  \textsc{May}(x, a_b, p) = \{c_2, c_3\}$. Then, 
 $\mathbb{E}_p(x)$ $= \{
 \langle c_1, c_2, c_2 \rangle, $ $ \langle c_1, c_2, c_3  \rangle, $ $ \langle c_2, c_2, c_2  \rangle,$ $  \langle c_2, c_2, c_3  \rangle
 \}$.)
 % One particular combination in $\mathbb{E}_p(x)$ is the combination of the prediction labels $g^a(x)$ (for all $a \in \mathbb{A})$ of $x$.
 % By definition, the label receiving the largest count in this particular combination is $D_{\textit{MC}}(x)$.
Suppose $x_1$ and $x_2$ differ by $p$ and the label produced by $D_{\textit{MC}}(x_1)$ always receives the largest count (also considering tie-breaking cases) in every combination in $\mathbb{E}_p(x_1)$. 
 % In that case, the set of combinations will always yield $D_{\textit{MC}}(x_1)$ as the label receiving the largest count.
 % receiving the largest count in every combination it contains.
% We know that 
Recall that $D_{\textit{MC}}(x_2)$ for any $x_2$ from $x_1$ differ by $p$ should eventually produce a label that receives the largest count in $\langle g^a(x_2) \mid\forall a \in \mathbb{A} \rangle$, which must be one of the combinations in $\mathbb{E}_p(x_1)$. So, the label produced by $D_{\textit{MC}}(x_2)$ should be same as the label produced by $D_{\textit{MC}}(x_1)$. Lemma \ref{lem:malicious} captures this insight, which leads to Theorem~\ref{thm:dynamic_ensemble}.
 % by considering all patch regions.

\vspace{-1ex}
\begin{lem}\label{lem:malicious}
Let $x_1$ be a sample and $p$ represent a patch region. If the condition
$
\mathop{\arg\max}_{c}[\sum_{
c' \in e
}{1}[{c'=c}]]= D_{\textit{MC}}(x_1)
$ holds for all $e \in \mathbb{E}_p(x_1)$, called the \textbf{\emph{majority invariant}},
then $D_{\textit{MC}}(x_2) = D_{\textit{MC}}(x_1)$ holds for any pairs of $x_2$ and $x_1$ differing by $p$.
% \vspace{-0.5ex}
\end{lem}
% \vspace{-0.5ex}
\begin{comment}
\begin{lem}\label{lem:malicious}
Given a sample $x$, a set of smoothing schemes $S$, an ensemble smoothed model $G$ constructed by a set of smoothed models $\{F_{s},\ s \in S\}$, a universal set of prediction
labels $C$, a universal combinations set of local malicious labels $\Tilde{C}_{Mal}$ against the patch attack $A$ under constrain $M$ at a location $l$, if
\begin{equation}
\begin{aligned}
&\forall \Tilde{c}_{Mal} \in \Tilde{C}_{Mal}, \ 
G(x)=c_{clean},\ \\
&\mathop{argmax}\limits_{c \in C}[\sum\limits_{{c}_{Mal} \in \Tilde{c}_{Mal}}\mathbbm{1}[c={c}_{Mal}]]=c_{clean}
\end{aligned}
\end{equation}
Then, 
\begin{equation} 
\forall p \in A,\ G(x_{p(l,m)}) = c_{clean}
\end{equation}
\end{lem}
\end{comment}
\vspace{-2.2ex}
\begin{proof}
% Recall %from Definition~\ref{def:lm-label} that
% $g^a(\cdot)$ = $\mathop{\arg\max}_{c}\sum_{b\in \mathbb{B}(a,\cdot)} f^a_c(b)$.

By Lemma \ref{lem:no_malicious}, the label returned by $g^a(x_2)$ for all $x_2$ should be a label in the set $\textsc{May}(x_1, a, p)$.
Recall $\mathbb{E}_p(x)$ contains all possible combinations of the sets $\textsc{May}(x, a, p)$ and $g^a(x)=\mathop{\arg\max}_{c}\sum_{b\in \mathbb{B}(a,x)} f^a_c(b)$.
% and $D_{\textit{MC}(x_2)}$ includes one label returned by the function $g^a(x_2)$ for each $a \in \mathbb{A}$.
Then, we have
$
D_{\textit{MC}}(x_2)  = \mathop{\arg\max}\nolimits_{c}[\sum\nolimits_{a\in \mathbb{A}}%
% f^a_c(x) \text{ where } 
% \\
% f^a_c(x)  = 
\textsc{1}[[\mathop{\arg\max}\nolimits_{c'}\sum\nolimits_{b\in \mathbb{B}(a,x_2)} f^a_{c'}(b)]=c]]
% D_{\textit{MC}}(x_2)  = \mathop{\arg\max}_{c}\sum_{a\in \mathbb{A}}%
% % f^a_c(x) \text{ where } 
% % \\
% % f^a_c(x)  = 
% [\mathop{\arg\max}_{c}\sum_{b\in \mathbb{B}(a,x_2)} f^a_c(b)]
$
%\\
$=\mathop{\arg\max}_{c}[\sum_{a \in A}\textsc{1}[{g_{a}(x_2)=c}]]$, 
which is an element in the set
%\\
$\{\mathop{\arg\max}_{c}[\sum_{a \in A}\textsc{1}[{c'=c}]]\mid c'\in \textsc{May}(x_1, a, p)$\}
% \textsc{local}(g^a(x_1), p) \cup \{g_a(x_1)
%\\
$=\{\mathop{\arg\max}_{c}$$[\sum_{c' \in e}\textsc{1}[{c'=c}]]\mid e\in \mathbb{E}_p(x_1)\} $ (by the definition of $\mathbb{E}_p(x)$).
Recall the condition
$
\mathop{\arg\max}_{c}[\sum_{
c' \in e
}\textsc{1}[{c'=c}]]= D_{\textit{MC}}(x_1)
$ holds for all $e \in \mathbb{E}_p(x_1)$. So, 
% It implies all elements in
% $\{\mathop{\arg\max}_{c \in C}[\sum_{c' \in e}{1}_{c'=c}]\mid e\in \mathbb{E}_p(x_1)\}$ is $D_{\textit{MC}}(x_1)$.
we obtain $D_{\textit{MC}}(x_2)=D_{\textit{MC}}(x_1)$.
   \end{proof}

\vspace{-0.3cm}
\begin{thm}[majority invariant analysis]\label{thm:dynamic_ensemble}
Given a sample $x$,
if Lemma~\ref{lem:malicious} holds for all patch regions, then $x$ is certifiably robust on $D_{\textit{MC}}$.
\end{thm}
% \vspace{-3ex}
\begin{proof}
{ 
It directly follows Lemma \ref{lem:malicious}.
 %and the definition of  patch robustness}.%
}
\end{proof}
\vspace{-0.1cm}
 \textsc{MajorityInvariantCertification}(.) in Alg.~\ref{alg:MajorCert} enumerates all label combinations in $\mathbb{E}_p(x)$ (generated from the sets $V[a]$ for all $a \in \mathbb{A}$ and the set of patch regions based on a given patch size), checks whether the majority invariant condition 
in Lemma~\ref{thm:dynamic_ensemble} 
holds for each combination in $\mathbb{E}_p(x)$ for all patch regions, and reports \textit{True} if that is the case, otherwise \textit{False}. 
It certifies samples even if Theorem \ref{thm:static_ensemble} cannot apply.
We note that the function is able to report a sample is certifiably robust even if the majority of DRS defenders fail to do so.  

\section{Evaluation}
\label{sec:expt}
\subsection{Setup}
We implement \textit{MajorCert} on Pytorch 1.13.0 
%\cite{paszke2019pytorch} 
and conduct the case study 
%All experiments were performed 
on a Ubuntu 20.04 server with a 48-core 3.0GHz Xeon CPU, 256GB RAM, and 4 2080Ti GPU cards. The implementation of \textit{MajorCert} is available in \cite{my}.
% with the VRAM size of 11GB.

\textbf{Datasets}. We adopt the CIFAR10 dataset\,\cite{krizhevsky2009learning}: 50000 training and 10000 test samples.
We divide the training samples into our training and validation sets using
a random 4:1 split.
% We randomly divided the downloaded training samples into 40,000 training samples as our training dataset and 10,000 validation samples. 

% hich has 50,000 training samples and 10,000 test samples. 

\textbf{Ablation strategy}. We adopt the row, column, and block ablation strategies of DRS\cite{levine2020randomized} to produce three DeRandomized Smoothing (DRS) defenders. The hyperparameters follow the optimal ones in \cite{levine2020randomized}: logits abstention threshold = 0.3%
% \todo[inline]{I add threshold here because the reason why PG shows more improvement than DRS in his paper maybe he uses a DRS with threshold 0.2 to compare his PG, which I only found in his program but do not shown in his paper. threshold have other meaning in PG's paper for his masking. (or we can ignore to say it since we need to spend space to explain it)}
%(if a logit value is greater than then the threshold, then regard it as a vote), 
,
column size = 4 and block size = 12. We set row size = column size due to their similarity. 
%adopt 4 as row size since row ablation is very similar to column ablation. 
We apply PatchGuard (PG) \cite{xiang2021patchguard} and \textit{MajorCert} based on these DRS defenders. 
(The code of PG \cite{pg} does not support the block strategy of DRS.)

\textbf{Models, techniques, and  hyperparameters}. We adopt the ResNet18 \cite{he2016deep} implementation in the DRS repository \cite{drs} as the model architecture, the original implementations and hyperparameters of DRS \cite{drs} and PG-DRS \cite{pg}, and the training script of DRS \cite{drs}.
% including the use of early-stop on the validation set during the training process to avoid over-fitting. For training models for CIFAR-10, we set
We set the number of epochs and batch size to 450 and 128. 
The learning rate is set to 0.1 and reduced by a factor of 10 at epochs \{150, 250, 350\}. 
% Other hyperparameters followed their work.
%\todo[inline]{add justifications of choosing 2x2 and 5x5}
% We use stochastic gradient descent with momentum as our optimizer, where momentum is 0.9 and L2 weight penalty is 0.0005.

\textbf{Evaluation metrics}. 
%We measure the clean accuracy and certified robust accuracy on the test dataset. 
\textit{Clean accuracy} is the fraction of test samples with correct prediction labels.
% (e.g., the label output by  $D_{\textit{MC}}$(.)). 
 \textit{Certified robust accuracy} is the fraction of test samples whose prediction labels are correct and are certifiably robust
 % by a technique
 against any square patch of size $m$. We study $m \times m$ = $2\times2$ and $= 5\times5$, which are consistent with previous works \cite{levine2020randomized,xiang2021patchguard}.
\begin{table}[]
\caption{clean accuracy and certified robust accuracy in CIFAR10}
\setlength\tabcolsep{2.5pt}
\centering
\label{table:main_re}
\begin{tabular}{|c|cc|cc|cc|cc|}
\hline
Patch size & \multicolumn{2}{c|}{$2\times2$}           & \multicolumn{2}{c|}{$5\times5$}           & \multicolumn{2}{c|}{\multirow{2}{*}{Mean Clean}} & \multicolumn{2}{c|}{\multirow{2}{*}{Mean Robust}} \\ \cline{1-5}
Metric     & \multicolumn{1}{c|}{Clean}           & Robust          & \multicolumn{1}{c|}{Clean}           & Robust          & \multicolumn{2}{c|}{}                            & \multicolumn{2}{c|}{}                             \\ \hline
PG-DRS \cite{xiang2021patchguard}& \multicolumn{1}{c|}{84.7\%}          & 69.2\%          & \multicolumn{1}{c|}{84.6\%}          & 57.7\%          & \multicolumn{2}{c|}{84.7\%}                      & \multicolumn{2}{c|}{63.5\%}                       \\ \hline
DRS \cite{xiang2021patchguard}       & \multicolumn{1}{c|}{83.9\%}          & 68.9\%          & \multicolumn{1}{c|}{83.9\%}          & 56.2\%          & \multicolumn{2}{c|}{83.9\%}                      & \multicolumn{2}{c|}{62.6\%}                       \\ \hline
PG-BN \cite{xiang2021patchguard}     & \multicolumn{1}{c|}{84.5\%}          & 63.8\%          & \multicolumn{1}{c|}{83.9\%}          & 47.3\%          & \multicolumn{2}{c|}{84.2\%}                      & \multicolumn{2}{c|}{55.6\%}                       \\ \hline
IBP \cite{han2021scalecert}       & \multicolumn{1}{c|}{65.8\%}          & 51.9\%          & \multicolumn{1}{c|}{47.8\%}          & 30.3\%          & \multicolumn{2}{c|}{56.8\%}                      & \multicolumn{2}{c|}{41.1\%}                       \\ \hline
CBN \cite{han2021scalecert}       & \multicolumn{1}{c|}{83.2\%}          & 51.0\%           & \multicolumn{1}{c|}{83.2\%}          & 16.2\%          & \multicolumn{2}{c|}{83.2\%}                      & \multicolumn{2}{c|}{33.6\%}                       \\ \hline 
\hline
MC\tablefootnote{MC will achieve the same accuracy as this row if \textsc{MajorityCertification} in line 7 of Algorithm \ref{alg:MajorCert} is disabled. }         & \multicolumn{1}{c|}{\textbf{88.1\%}} & \textbf{71.1\%} & \multicolumn{1}{c|}{\textbf{88.1\%}} & \textbf{57.8\%} & \multicolumn{1}{c|}{\textbf{88.1\%}}         & 5.4x        & \multicolumn{1}{c|}{\textbf{64.5\%}}         & 11x         \\ \hline
{PG-DRS\tablefootnote{\label{explain}We select PG-DRS with best certified robust accuracy and its corresponding results on DRS.}}     & \multicolumn{1}{c|}{87.3\%}          & 69.7\%          & \multicolumn{1}{c|}{84.5\%}          & 57.2\%          & \multicolumn{1}{c|}{85.9\%}         & 1x         & \multicolumn{1}{c|}{63.5\%}         & 1x          \\ \hline
DRS\textsuperscript{\ref{explain}}        & \multicolumn{1}{c|}{86.6\%}          & 69.6\%          & \multicolumn{1}{c|}{84.2\%}          & 57.1\%          & \multicolumn{1}{c|}{85.4\%}         & -          & \multicolumn{1}{c|}{63.4\%}         & -           \\ \hline
\end{tabular}

\vspace{-3ex}
\end{table}

% \begin{figure}[h]

% \centering
% % \hspace{20cm}
% % \hspace{2cm}
% \hspace{0.5cm}
% \includegraphics[width=8cm]{cifar.jpg}
% \caption{CIFAR-10 clean accuracy ($x$-axis) and certified robust accuracy ($y$-axis) achieved by MC, PG, and DRS. % with different ablation strategies. %The left shows their performance against a $5\times5$ patch and the right one shows their performance against a $2\times2$ patch.
% }

% \label{fig:CIFAR10-2t2and5t5}
% % \vspace{-2ex}
% \end{figure}

\subsection{Results and Data Analysis}
Table \ref{table:main_re} shows the patch robustness certification results achieved by \textit{MajorCert} (MC), DRS \cite{levine2020randomized}, and PG-DRS \cite{xiang2021patchguard} (the last three rows) in our case study, and shows the best results of DRS, PG, Interval Bound Propagation (IBP) \cite{chiang2020certified}, and Clipped BagNet (CBN) \cite{zhang2020clipped} in the literature. % shown in the table.
%  For our models with more than one smoothing scheme, we chose the PG-DRS model, which has the highest certified accuracy, and its corresponding DRS model to show in the table. 

In Table  \ref{table:main_re}, MC outperforms all peer techniques in clean and certified robust accuracy.
Both MC and PG develop atop DRS. 
MC achieves 3.6\% higher clean accuracy and 0.6\% higher certified robust accuracy against a $5\times5$ patch, and 0.8\% higher clean accuracy and 1.4\% higher certified robust accuracy against a $2\times2$ patch than PG.
Suppose we normalize the accuracy improvement of PG atop DRS in our case study to 1.
The improvements achieved by MC are 5.4x in clean accuracy and  11x in certified robust accuracy.  
From Table  \ref{table:main_re}, advancing the state in patch robustness certification appears to be a hard problem since PG can only slightly improve DRS. 
% in accuracy.  

Fig. \ref{fig:CIFAR10-2t2and5t5} further shows the breakdown by ablation strategy used by PG and DRS in our case study. 
PG's improvement over DRS is inconsistent across ablation strategies. Across the two plots, PG is more effective than DRS using row ablation but not column ablation. DRS is always the most effective when using block ablation in clean accuracy but not certified robust accuracy.
%, but PG cannot build on this ablation strategy due to its algorithmic limitation.
% \todo[inline]{ PG's algorithm is OK to apply block smoothing,, but it hasn't been implemented in their program.}
MC builds on top of all the above DRSs and consistently outperforms all peer techniques in the case study.
% \todo[inline]{ Can we say that we observe the same order on the skin cancer dataset, but due to page limit, we leave reporting it in future work?\\ In DermaMNIST, we have high certified accuracy but low clean accuracy, while PG has a clean high accuracy but low certified accuracy. }
\begin{figure}[t]

\centering
% \hspace{20cm}
% \hspace{2cm}
\hspace{0.5cm}
\includegraphics[width=8cm]{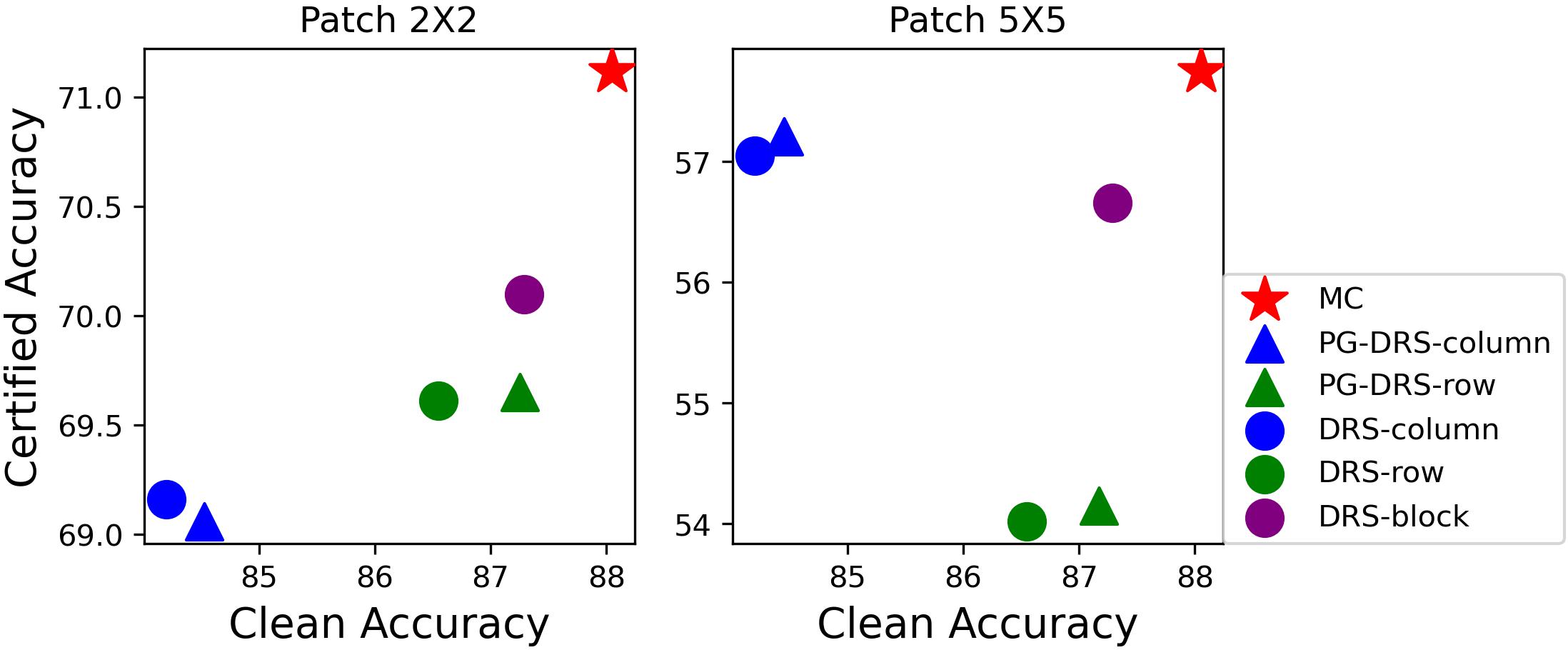}
\caption{CIFAR-10 clean accuracy ($x$-axis) and certified robust accuracy ($y$-axis) achieved by MC, PG, and DRS. % with different ablation strategies. %The left shows their performance against a $5\times5$ patch and the right one shows their performance against a $2\times2$ patch.
 \vspace{-1ex}
}

\label{fig:CIFAR10-2t2and5t5}
% \vspace{-2ex}
\end{figure}
\vspace{-1.2ex}
\section{Related Work}
\label{sec:relatedWork}
 
IBP \cite{chiang2020certified} is the first certified defense against patches.
%by a bounding technique  \cite{gowal2019scalable}\todo{why cite this for IBP?}, 
% but is computationally expensive and weak in certified accuracy.
DRS \cite{levine2020randomized} predicts a label for each ablation as the vote. % from the ablation. 
% It proves a common safety margin for the label receiving the largest votes not affected by any possible patch bounded by all patch regions. 
It then uses the label receiving the largest votes as output and proves a common safety margin for certification. If the margin of votes between the largest and second largest is large enough, no patch within a given bound can affect the prediction label.
PatchGuard \cite{xiang2021patchguard} is built atop either DRS \cite{levine2020randomized} or BagNet \cite{brendel2019approximating}. 
It constructs a heuristic mask on prediction vectors to mask out high values in those vectors to empirically improve the accuracy.
% Unlike DRS,
It then aggregates the vectors after masking to produce the prediction label and scans for each patch region to check whether the malicious label exists similar to MajorCert. But different from MajorCert, if it finds any malicious labels, it cannot certify the scanned sample.
The work of Salman et al. \cite{salman2022certified}  combines the certification analysis of DRS with Vision Transformer architecture (ViT) instead of ResNet architecture to achieve higher accuracy empirically.
As noted in \cite{li2022vip}, further analysis is required to reveal why ViT outperforms ResNet in this problem. 
PatchCleanser 
\cite{xiang2022patchcleanser} connects certified detection and certified recovery with a two-stage certification analysis. VIP \cite{li2022vip} is also a ViT-based proposal for both certified detection and certified recovery, but still following the line of DRS to 
find a common margin for certified recovery.
% Unlike our technique,
% it requires detailed knowledge of the patch size and shape during the training process.
% rather than checking the consistency across ablations overlapping with individual patch regions. 
% However, computation complexity
% to achieve state-of-the-art performance. 
% Unlike our technique,
% it requires detailed knowledge of the patch size during the training process.
% whereas our method does not make this assumption.

% We need to be aware of existing work in certified patch robustness to apply location-wise consistency detection between multi-smoothing. To the best of our knowledge, MajorCert contributes as the first technique of this kind. However, we represent a preliminary exploration of it, more research, like the relationship of different ablation strategies or the training strategy for maximizing the robustness using the information between multiple smoothing, still requires further research.

\vspace{-0.2ex}
\section{Conclusion}
\label{sec:conclusion}
This paper has presented \textit{MajorCert}, a novel patch robustness certification technique.
% The basic idea of \textit{MajorCert} is to determine whether all combinations of all original and malicious labels of all ablations from different base models overlapping with each patch region result in the same prediction label under the majority rule.  
We have proven its correctness 
% with theorems 
and testified its effectiveness through a case study.
% We leave the generalization of  \textit{MajorCert} to handle a wider class of majority invariants with experimentation as future work.

\section*{Acknowledgment}
This research is partly supported by the CityU MF\_EXT (project no. 9678180). 
W.K. Chan is the corresponding author.
\bibliographystyle{ieeetr}
\bibliography{myrefs}

\end{document}